\documentclass[letterpaper, 10 pt, conference]{ieeeconf}  

\IEEEoverridecommandlockouts                              

\usepackage{color}
\usepackage{multicol}
\usepackage[bookmarks=true]{hyperref}

\usepackage{gensymb}
\usepackage[pdftex]{graphicx}
\usepackage[cmex10]{amsmath}
\usepackage[ruled,vlined]{algorithm2e}
\usepackage{array}
\usepackage[caption=false,font=footnotesize]{subfig}
\usepackage{url}

\usepackage{amsthm}
\usepackage{amssymb}
\usepackage{listings}

\graphicspath{{./pics/}}
\DeclareGraphicsExtensions{.pdf,.jpeg,.png}

\newtheorem{thm}{Theorem}
\newtheorem{corollary}{Corollary}[thm]

\theoremstyle{definition}
\newtheorem{defn}{Definition}

\theoremstyle{exampstyle} \newtheorem{example}{Example}
\theoremstyle{exampstyle} \newtheorem{problem}{Problem}

\title{\LARGE \bf Formal Design of Robot Integrated Task and Motion Planning*
}

\author{Rafael Rodrigues da Silva$^{1,2}$, Bo Wu$^{1}$ and Hai Lin$^{1}$
\thanks{*The financial supports from NSF-CNS-1239222, NSF-CNS-1446288  and NSF- EECS-1253488 for this work are greatly acknowledged.}
	\thanks{$^{1}$ All authors are with Department of Electrical Engineering, University of Notre Dame, Notre Dame, IN 46556, USA.
		{\tt\small (rrodri17@nd.edu;~bwu3@nd.edu;~hlin1@nd.edu)}}
\thanks{$^{2}$ The first author would like to appreciate the scholarship support by CAPES/BR, BEX 13242/13-0}
}

\begin{document}

\maketitle
\thispagestyle{empty}
\pagestyle{empty}

\begin{abstract}

Integrated Task and Motion Planning (ITMP) for mobile robots becomes a new trend.  Most existing methods for ITMP either restrict to static environments or lack performance guarantees. This motivates us to use formal design methods for mobile robot’s ITMP in a dynamic environment with moving obstacles. Our basic idea is to synthesize a global integrated task and motion plan through composing simple local moves and actions, and to achieve its performance guarantee through modular and incremental verifications. The design consists of two steps. First, reactive motion controllers are designed and verified locally. Then, a global plan is built upon these certified controllers by concatenating them together. In particular, we model the controllers and verify their safety through formulating them as Differential Dynamic Logic (dL) formula. Furthermore, these proven safe controllers are abstracted in Counter Linear Temporal Logic over Constraint System CLTLB($\mathcal{D}$) and composed based on an encoding to Satisfiability Modulo Theories (SMT) that takes into account the geometric constraints. Since dL allows compositional verification, the sequential composition of the safe motion primitives also preserves safety properties. Illustrative examples are presented to show the effectiveness of the method.


\end{abstract}

\section{Introduction}

Traditionally, task and motion planning for mobile robots are designed separately, and they work in a  hierarchical manner with a task planner sitting on top of motion planners \cite{latombe2012robot}. Task planning is usually carried on symbolically based on an abstracted view of physical environments that ignores details in geometric or physical constraints. Hence, it is possible that there is no feasible trajectories to achieve the derived mission plans. Therefore, a recent trend is towards an Integrated Task and Motion Planning (ITMP), see e.g.,   \cite{cambon2003overview,galindo2004improving,wolfe2010combined,kaelbling2011hierarchical,kimmel2012pracsys, saha2014automated,lin2014mission,hung2014motion,nedunuri2014smt,garrett2015ffrob} and references therein.


%

Earlier efforts in ITMP, such as Asymov \cite{cambon2003overview} and SMAP  \cite{plaku2010sampling}, were still based on abstractions of the working environment and used a symbolic planner to provide a heuristic guidance to the motion planner. Recent work, such as \cite{littlefield2014extensible} and \cite{dornhege2012semantic},  introduced  a ``semantic attachment,"  i.e. a predicate that is solved by a motion planner, to the symbolic planner.  An overview of the recent developments in the symbolic motion planning can be found in \cite{lin2014mission}, where the task planning problem is reduced to model checking. Since these methods are based on abstracted symbolic models of the environments, it is a common assumption that the working environment is known or static and the robot is the only moving object (or the robot itself carries  other movable objects). However, in practice, a robot often shares its workspace with others robots or even humans, and the environment often changes over time in a way that is hard to predict. This motivates us to investigate the ITMP problem for mobile robots  in a dynamic environment with moving obstacles.

Inspired by behavior based robotics \cite{nakhaeinia2013behavior}, we adopt a hierarchical planner consisting of two layers:  global and local. Our basic idea is to synthesize a global and integrated task and motion plan through composing simple local moves and actions, and to achieve its performance guarantee through modular incremental verifications. The design consists of two steps. First, basic motion primitives are designed with verified performances. Then, a global plan is built upon these certified motion primitives. Since the method proposed here is of bottom-up and compositional nature, so we call it as CoSMoP (Composition of Safe Motion Primitives). 

In the first step, we propose to use a formally verified motion controllers that we call safe motion primitives. These primitives are designed offline, modeled and verified in Differential Dynamic Logic (d$\mathcal{L}$) \cite{platzer2010logical}, for which verification software tools are available, e.g., KeYmaera \cite{platzer2010logical}. In particular, we use the Dynamic Window Approach (DWA) \cite{fox1997dynamic}  as obstacle avoidance motion primitives in this paper. DWA is a widely adopted and efficient approach for mobile robots to avoid collisions in uncertain and dynamic environments.  The safety of an extended DWA on collision avoidance for moving obstacles has been formally proved in \cite{mitsch2013provably} using d$\mathcal{L}$ and hybrid system verification. With this proof, we can abstract them to the global layer, where the task and motion plans are integrated. 

In the second step, those safe motion primitives are encoded to a Satisfiability Modulo Theories (SMT) solver as motion primitive constraints. This layer synthesizes a composition of pairs of actions (i.e. safe motion primitives) and waypoints (i.e. terminal positions), which is the sequential execution of actions that the robot must perform to ensure a task specification formally. The CoSMoP encodes an ITMP problem to the SMT by extending the Bounded Satisfiability Checking (BSC) \cite{pradella2013bounded} and using the Counter Linear Temporal Logic over Constraint System CLTLB($\mathcal{D}$) \cite{bersani2010bounded} language, a Linear Temporal Logic (LTL) extension. The BSC models consist of temporal logic rather than transition systems; thus, the problem encoding is more compact and elegant. Moreover, it was also shown that if the constraint system $\mathcal{D}$ is decidable, then so is the CLTLB($\mathcal{D}$), and it can be encoded to SMT \cite{bersani2010bounded}. Therefore, encoding the ITMP problem using CLTLB($\mathcal{D}$) language allows the description of a wide range of system properties in a problem that is decidable. 

In summary, the contribution of this work is to provide an automatic synthesis that is provably safe even for unexpected moving obstacles that, to the best of our knowledge, has not been attempted before for ITMP:
\begin{itemize}
	\item Unlike \cite{nakhaeinia2013behavior}, where the motion plan is not formally verified, in our approach, the performance of the resulting integrated task and motion plan is formally guaranteed. 
	\item Unlike \cite{saha2014automated, lin2014mission}, we do not assume static environment and complete knowledge of the other moving agents that is required for verification of symbolic partitioned environments.
	\item Unlike \cite{nedunuri2014smt, littlefield2014extensible, lin2014mission, garrett2015ffrob}, we do not assume a static environment where the robot is the only moving agent, which is assumed in these others ITMP approaches.
	\item Unlike \cite{bouraine2012provably,nakhaeinia2013behavior,mitsch2013provably,althoff2014online,hess2014formal}, where only motion specifications are considered, we combine task and motion specifications that allow specifications such as moving objects in the environment. 
	\item Unlike \cite{pradella2013bounded,bersani2010bounded}, we use CLTLB($\mathcal{D}$) for automatic synthesis instead of model checking.
\end{itemize}

 The rest of the paper is organized as follows.  Section \ref{sec:preliminaries} presents some background for understanding CoSMoP approach and defines the scenario used in this work. Section  \ref{sec:CoSMoP} introduces the CoSMoP design procedure and formulates the problem. Section \ref{sec:prim} presents the design of motion primitives for the scenario proposed here. Section \ref{sec:comp} presents how to synthesize a global and integrated plan using SMT solver. Section \ref{sec:sim} studies which parameters affect the execution time the most. Section \ref{sec:conclusion} concludes the paper with a discussion and proposes possible future works.

\section{Preliminaries}\label{sec:preliminaries}

\subsection{Differential Dynamic Logic}

The Differential Dynamic Logic $d\mathcal{L}$ verifies a symbolic hybrid system model, and, thus, can assist in verifying and finding symbolic parameters constraints. Most of the time, this turns into an undecidable problem for model checking \cite{platzer2010logical}. Yet, the iteration between the discrete and continuous dynamics is nontrivial and leads to nonlinear parameter constraints and nonlinearities in the dynamics. Hence, the model checking approach must rely on approximations. On the other hand, the $d\mathcal{L}$ uses a deductive verification approach to handling infinite states, it does not rely on finite-state abstractions or approximations, and it can handle those nonlinear constraints. 

The hybrid systems are embedded to the d$\mathcal{L}$ as hybrid programs, a compositional program notation for hybrid systems. 

\begin{defn}[Hybrid Program]
	A hybrid program \cite{platzer2010logical} ($\alpha$ and $\beta$) is defined as:
	\begin{equation*}
	\alpha, \beta ::= \begin{cases}
	x_1 := \theta_1,...,x_n:=\theta_n \mid  ?\chi \mid \alpha ; \beta \mid \alpha \cup \beta \mid \alpha^* \mid \\
	x_1^{\prime} := \theta_1,...,x_n^{\prime}:=\theta_n \& \chi			
	\end{cases}  
	\end{equation*}
	where:
	\begin{itemize}
		\item $x$ is a state variable and $\theta$ a first-order logic term.
		\item $\chi$ is a first-order formula.
		\item $x_1 := \theta_1,...,x_n:=\theta_n$ are discrete jumps, i.e. instantaneous assignments of values to state variables. 
		\item $x_1^{\prime} := \theta_1,...,x_n^{\prime}:=\theta_n \& \chi$ is a differential equation system that represents the continuous variation in system dynamics. $x_i^{\prime} := \theta_i$ is the time derivative of state variable $x_i$, and $\& \chi$ is the evolution domain. 
		\item $?\chi$ tests a first-order logic at current state.
		\item $\alpha ; \beta$ is a sequential composition, i.e. the hybrid program $\beta$ will start after $\alpha$ finishes. 
		\item $\alpha \cup \beta$ is a nondeterministic choice.
		\item $\alpha^*$ is a nondeterministic repetition, which means that $\alpha$ will repeat for finite times. 
	\end{itemize}
\end{defn}

Thus, we can define the $d\mathcal{L}$ formula, which is a first-order dynamic logic over the reals for hybrid programs.

\begin{defn}[$d\mathcal{L}$ formulas]
	A $d\mathcal{L}$ formula \cite{platzer2010logical} ($\phi$ and $\psi$) is defined as:
	\begin{equation*}
	\phi, \psi ::= \chi \mid \neg \phi \mid \phi \wedge \psi \mid \forall x \phi \mid \exists x \phi \mid [\alpha] \phi \mid \langle \alpha \rangle \phi
	\end{equation*}
	where:
	\begin{itemize}
		\item $[\alpha] \phi$ holds true if $\phi$ is true after all runs of $\alpha$.
		\item $\langle \alpha \rangle \phi$ holds true if $\phi$ is true after at least one runs of $\alpha$.
	\end{itemize}
\end{defn}

$d\mathcal{L}$ uses a compositional verification technique that permits the reduction of a complex hybrid system into several subsystems \cite{platzer2010logical}. This technique divides a system $\psi \rightarrow [\alpha] \phi$ in an equivalent formula $\psi_1 \rightarrow [\alpha_1] \phi_1 \wedge \psi_2 \rightarrow [\alpha_2] \phi_2$, where each $\psi_i \rightarrow [\alpha_i] \phi_i$ can be proven separately. In our approaches we use this technique backwards, we prove a set of $d\mathcal{L}$ formulas $\psi_i \rightarrow [\alpha_i] \phi_i$, where each one is the $i^{th}$ motion primitive model, and we use the SMT to compose an equivalent $\psi \rightarrow [\alpha] \phi$ that satisfies a mission task. Therefore, the synthesized hybrid system performance is formally proven. 

%
%

\subsection{Counter Linear Temporal Logic Over Constraint System}

We express the specification of an autonomous mobile robot using Counter Linear Temporal Logic Over Constraint System CLTLB($\mathcal{D}$) defined in \cite{bersani2010bounded}. This language is interpreted over Boolean terms $p \in AP$ or arithmetic constraints $R \in \mathcal{R}$ belong to a general constraint system $\mathcal{D}$, where $AP$ is a set of atomic propositions and $\mathcal{R}$ is a set of arithmetic constraints. Thus, the semantics of a CLTLB($\mathcal{D}$) formula is given in terms of interpretations of a finite alphabet $\Sigma \in \{AP, \mathcal{R}\}$ on finite traces over a finite sequence $\rho$ of consecutive instants of time with length $K$, meaning that $\rho(k)$ is the interpretation of $\Sigma$ at instant of time $k \in \mathcal{N}_{\rho}, \mathcal{N}_{\rho} = \{0,...,K\}$. Moreover, the arithmetic terms of an arithmetic constraint $R \in \mathcal{R}$ are variables $x$ over a domain $D \in \{\mathbb{Z}, \mathbb{R}\}$ valuated at instants $i$ and, thus, are called arithmetic temporal terms \textit{a.t.t.},   

\begin{defn}[Arithmetic Temporal Term]
	A CLTLB($\mathcal{D}$) arithmetic temporal term (\textit{a.t.t.}) $\varphi$ is defined as:
	\begin{equation*}
	\varphi ::= x \mid \bigcirc \varphi \mid \bigcirc^{-1} \varphi
	\end{equation*}
	where $\bigcirc$ and $\bigcirc^{-1}$ stands for next and previous operator.
\end{defn}

Therefore, a CLTLB($\mathcal{D}$) formula is a LTL formula over the \textit{a.t.t.} defined as below.

\begin{defn}[Formula]
	A CLTLB($\mathcal{D}$) formula ($\phi$, $\phi_1$ and $\phi_2$) is defined as,  
	\begin{equation*}
	\phi, \phi_1, \phi_2 ::= 
	\begin{cases}
	p \mid R(\varphi_1,...,\varphi_n) \mid \neg \phi \mid \phi_1 \wedge \phi_2 \mid \\ 
	\bigcirc \phi  \mid \bigcirc^{-1} \phi \mid \phi_1 \mathbf{U} \phi_2 \mid \phi_1 \mathbf{S} \phi_2
	\end{cases}
	\end{equation*}
	where,
	\begin{itemize}
		\item $p \in AP$ is a atomic proposition, and $R \in \mathcal{R}$ is a relation over the \textit{a.t.t.} such as, for this work, we limit it to linear equalities or inequalities, i.e. $R(\varphi_1,...,\varphi_n) \equiv \sum_{i=1}^{n} c_i \cdot \varphi_i \# c_0$, where $\# \equiv \langle =, <, \leq, >, \geq \rangle$ and  $c_i, \varphi_i \in D$. 
		\item $\bigcirc$, $\bigcirc^{-1}$, $\mathbf{U}$ and $\mathbf{S}$ stands for usual next, previous, until and since operators on finite traces, respectively.
	\end{itemize}
\end{defn}

Based on this grammar, it can also use others common abbreviations, including:

\begin{itemize}
	\item Standard boolean, such as $true$, $false$, $\vee$ and $\rightarrow$.
	\item $\Diamond \phi$ that stands for $true \mathbf{U} \phi$, and it means that $\phi$ eventually holds before the last instant (included). 
	\item $\square \phi$ that stands for $\neg \Diamond \neg \phi$, and it means that $\phi$ always holds until the last instant. 
	\item $Last [\phi]$ that stands for $\Diamond (\neg \bigcirc true) \wedge \phi$, where $\neg \bigcirc true$ on finite trace is only $true$ at last instant. Thus, it means that $\phi$ is true at the last instant of the sequence $\rho$. 
\end{itemize}

A CLTLB($\mathcal{D}$) formula is verified in a Bounded Satisfiability Checking (BSC) \cite{pradella2013bounded}. Hence, it is interpreted on a finite sequence $\rho$ with length $K$. Therefore, $\rho(k) \vDash p$ means that $p$ holds true in the sequence $\rho$ at instant $k$ ($p \vdash \rho(k)$).

\begin{defn}[Semantics]
	The semantics of a CLTLB($\mathcal{D}$) formula $\phi$ at an instant $k \in \mathcal{N}_{\rho}$ is as follow:
	
	\begin{itemize}
		\item $\rho(k) \vDash p \Longleftrightarrow p \vdash \rho(k)$.
		\item $\rho(k) \vDash R(\varphi_1,..., \varphi_n)  \Longleftrightarrow R(\varphi_1,..., \varphi_n) \vdash \rho(k)$.
		\item $\rho(k) \vDash \neg \phi \Longleftrightarrow \rho(k) \nvDash \phi$.
		\item $\rho(k) \vDash \phi_1 \wedge \phi_2 \Longleftrightarrow \rho(k)  \vDash \phi_1 \wedge \rho(k)  \vDash \phi_2$.
		\item $\rho(k) \vDash \bigcirc \phi \Longleftrightarrow \rho(k+1)  \vDash \phi$.
		\item $\rho(k) \vDash \bigcirc^{-1} \phi \Longleftrightarrow \rho(k-1)  \vDash \phi$.
		\item $\rho(k) \vDash \phi_1 \mathbf{U} \phi_2 \Longleftrightarrow \begin{cases}
		\exists i \in [k,K]: \rho(i)  \vDash \phi_2 \wedge \\
		\forall j \in [k,i-1]: \rho(j) \vDash \phi_1
		\end{cases}$.
		\item $\rho(k) \vDash \phi_1 \mathbf{S} \phi_2 \Longleftrightarrow \begin{cases}
		\exists i \in [0,k]: \rho(i)  \vDash \phi_2 \wedge \\
		\forall j \in [i+1,k]: \rho(j) \vDash \phi_1
		\end{cases}$.
	\end{itemize}
\end{defn}

\subsection{Scene Description}

As a motivating example, we consider a building with two way doors that an assistant robot needs to move objects with its gripper. The robot shares its workspace with other robots and humans, and we call this scenario Clean Up. The robot should be able to find and move those objects to designated areas through doors while stay inside the workspace. 

In this scenario, each state robot state $q_r$ is a triple $\langle x, y, \alpha \rangle$ representing the robot pose in 2D, where $(x,y) \in \mathbb{Z}^2$ is the position in $mm$ and $\alpha \in \mathbb{R}$ is the heading angle in degrees. Yet, each object $j$ state $q_b^j$ is a triple $\langle x, y, p \rangle$ representing its 2D position $(x, y) \in \mathbb{Z}^2$ in $mm$ and a proposition $p \in AP$ that holds true when the robot is carrying this object. Based on a given scenario, scene description provides basic information on robots and the environment they work in. 

\begin{defn}[Scene Description] \label{def:scene}
    Scene description is a tuple $\mathcal{M} = \langle \mathcal{O}, \mathcal{D}, \mathcal{A}, \mathcal{B}, \mathcal{W} \rangle$:
	\begin{itemize}
		\item Obstacles $\mathcal{O}$: a set of rectangular obstacles in parallel to the axis $o_j \in \mathcal{O}: j \in N_{\mathcal{O}} = \{1,...,|\mathcal{O}|\}$ specified by two points $o_j = \langle (x_i, y_i), (x_f,  y_f) \rangle$ describing a pair of diagonal vertexes;
		\item Doors $\mathcal{D}$: a set of doors $d_j \in \mathcal{D}: j \in N_{\mathcal{D}} = \{1,...,|\mathcal{D}|\}$ that describe two robot poses $q_1$ and $q_2$ necessary to push and pass through this door, i.e. $d_j = \langle q_1, q_2 \rangle: q_1 = q_2 = \langle x, y, \alpha\rangle$. 
		\item Agent $a$: the robot $a = \langle l \rangle$ is abstracted as a square with length $l$.
		\item Objects $\mathcal{B}$: a set of movable objects $b_j \in \mathcal{B}: j \in N_{\mathcal{B}} = \{1,...,|\mathcal{B}|\}$ and $b_j = \langle l \rangle$ that is abstracted as a square with length $l$.
		\item Workspace $\mathcal{W} = \langle x, y, l \rangle$ : the workspace dimension description, which is assumed to be a square with center at position $( x, y)$ and length $l$.
	\end{itemize}
\end{defn}

 Now, we can define the scene description for this particular scenario as shown below.

\begin{example}\label{ex:example01}

 Consider the workspace shown in Fig. \ref{fig:example01}. This scene description has two obstacles $\mathcal{O}=\{\langle (-1500,-2500), (-1500,2500) \rangle, \langle (-1500,0), (2500,0) \rangle\}$ that refer to two walls that Door $1$ (also Door $2$) and Door $3$ are located. The Door $1$ is described by $d_1 = \langle (-2000, -500, 0^o), (-1000, -500, 180^o) \rangle$, the Door $2$ by $d_2 = \langle (-2000, 1000, 0^o), (-1000, 1000, 180^o) \rangle$, and the Door $3$ by $d_3 = \langle (-1000, 500, 270^o), (-1000, -500, 90^o) \rangle$. There are two objects abstracted as a square of $100mm$ (i.e. $b_1.l = b_2.l = 100$) initially at $q_b^1 = \langle 1900,-1000,false \rangle$ and $q_b^2 = \langle 2000,-1000,false \rangle$, where initially neither of the two objects is picked up. The robot is abstracted as a square of $400mm$ (i.e. $a.l = 400$) and starts at $q_r = \langle -2000, 0, 0.0 \rangle$. 

\begin{figure}[!t]
	\centering
	\includegraphics[width=2.5in]{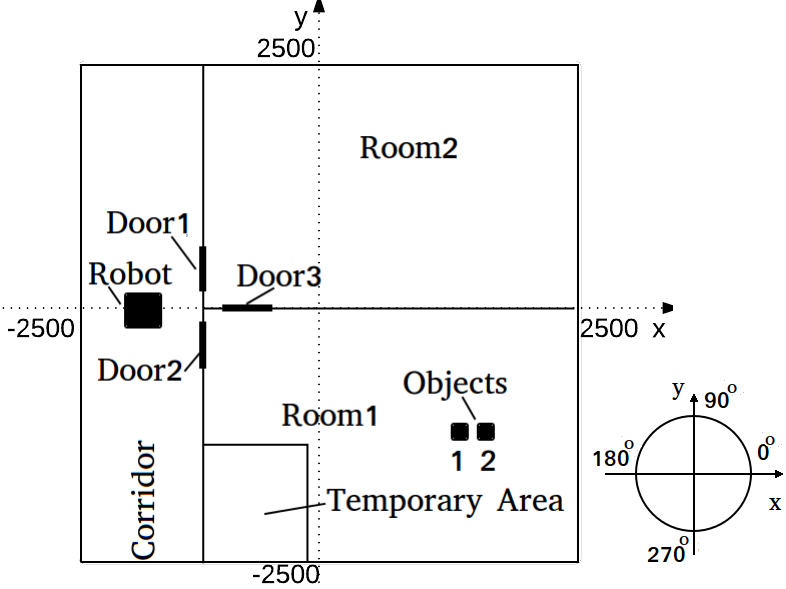}
	\caption{A example of blueprint for the clean up scenario.}
	\label{fig:example01}
\end{figure}
\end{example}

\section{CoSMoP framework}\label{sec:CoSMoP}

	
	This section describes a formal bottom-up approach that features a two-layer hierarchical ITMP architecture as shown in Fig. \ref{fig:framework}. The local layer implements reactive motion controllers such as the DWA that realize motion planning incorporating all kinematic and geometric constraints for dynamic environment with moving obstacles. These controllers are designed offline and are abstracted to the global layer as motion primitives specifications $\phi_{\mathcal{P}}(\mathcal{M})$, in CLTLB($\mathcal{D}$). 	The global layer generates a constraint system in the Constraint Generator based on scene description $\mathcal{M}$ and the task specification $\phi_{\mathcal{G}}$, in CLTLB($\mathcal{D}$), and encode it to a SMT solver. If this constraint system is satisfiable, a plan that is a roadmap for the local layer is extracted to the local layer. We assume that $\mathcal{M}$ have geometric details about the environment enough for global layer to search for a satisfiable plan. Let $Q_r$, $Q_b$ and $Q_{\pi}$ be a sequence of assigned values to robot and object states and assigned motion primitive at each instant $k \in \mathcal{N}_{\rho}$, respectively, we can formulate our problem as follows.   
	\begin{problem}
		Given a scene description $\mathcal{M}$, initial conditions, a task specification $\phi_{\mathcal{G}}$, the trace length $K$, design a set $\mathcal{U}$ of the reactive motion controllers in the local layer and respective motion primitives specifications $\phi_{\mathcal{P}}(\mathcal{M})$ and check if the specification $\phi_{\mathcal{G}}$ is satisfiable in the scene $\mathcal{M}$ using the controllers $\mathcal{U}$. If yes, find a trace $s$ with length $K$, where $s(k) = \langle q_{r}(k), \pi(k) \rangle$  at instant $k \in \mathcal{N} = \{1,...,K\}$, $q_{r}(k) \in Q_r$ is a robot state and $\pi(k) \in Q_{\pi}$ is a motion primitives at instant $k$ such as $\pi(k)$ defines what controller $u \in \mathcal{U}$ to take at $q_{r}(k-1)$ to go to $q_{r}(k)$.
	\end{problem}
	\begin{figure}[!t]
		\centering
		\includegraphics[width=3.5in]{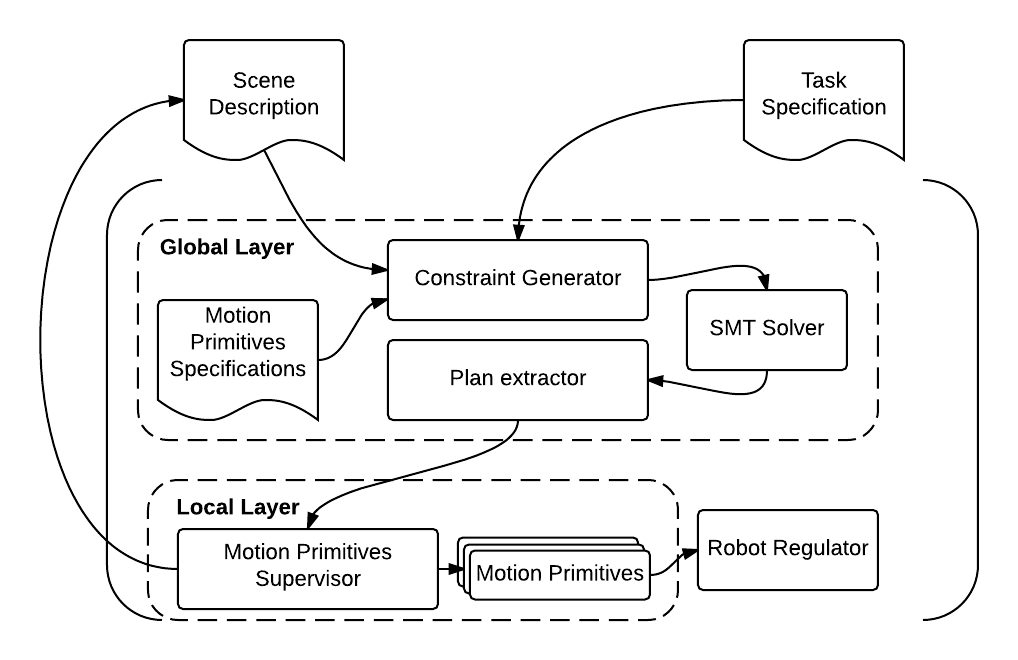}
		\caption{CoSMoP framework.}
		\label{fig:framework}
	\end{figure}
	
	Note that we are restricted to take at most $K$ actions for the task, therefore it is a bounded time planning problem. The motion primitives are abstract actions that a robot can execute, such as moving to some place, picking up objects and so on. 
	\begin{defn}[Motion Primitive]
		The Motion Primitive $\Pi \in \mathcal{P}$ is formally defined as $\langle u, \phi_{\Pi} \rangle$, where $\mathcal{P}$ is a set of abstract actions available to the global layer. The symbol $u \in \mathcal{U}$ is a reactive motion controller. The specification $\phi_{\Pi}$ constrains the states $q(k-1)$ and $q(k)$ such as $q(k) \in [q_{r}(k) \in Q_r] \cup [q_{b}^j(k) \in Q_b] \cup [\pi(k) \in Q_{\pi}]$ based on the Scene Description $\mathcal{M}$ when the robot takes $u$.
	\end{defn}

	
	Basically the framework needs to solve three subproblems. First, we design reactive motion controllers $u_i \in \mathcal{U}: i \in \mathcal{N}_{\mathcal{U}} = \{1,...,|\mathcal{N}_{\mathcal{U}}|\}$ to provide enough maneuvers for the robot to complete the given task. If it is needed, the design criterion is that each controller modeled as a hybrid program $\alpha_{u_i}$ must formally ensure a safety property $\phi_{u_i}^{safe}$ after any execution, if the initial state is safe and satisfies the initial condition $\phi_{u_i}^{initial}$, i.e. $\phi_{u_i}^{initial} \wedge \phi_{u_i}^{safe} \rightarrow [\alpha_{u_i}] \phi_{u_i}^{safe}$ in $d\mathcal{L}$. 
	
	Second, we abstract these controllers from the local to the global layer, i.e. to design the specification $\phi_{\Pi}$, in CLTLB($\mathcal{D}$) formulas, to enforce safety requirement after concatenation of designed motion primitives. We do so by ensuring that for any plan $s$ of size $K$, where $\forall \pi(k) \in Q_{\pi}: \pi(k) \in \mathcal{P} , u(k) \in \pi(k)$, the following two conditions hold.
	\begin{itemize}
		\item For each $i\in \mathcal{N}$, $\phi_{u(k)}^{safe}$ is satisfiable for at least one trajectory between  $q_{r}(k-1)$ and $q_{r}(k)$.
		\item For each $i\in \mathcal{N}$, $q_{r}(k-1) \vDash \phi_{u(k)}^{initial}$.
		
	\end{itemize} This specification depends on the scene description $\mathcal{M}$ and the conjunction of all specifications $\phi_{\Pi}$ is called the motion primitives specification $\phi_{\mathcal{P}}(\mathcal{M})$. Since the output of global layer is a sequence $s$, where each $\pi(k) \in Q_{\pi}$ assigns one of $\Pi \in \mathcal{P}$. If $\phi_{\mathcal{P}}(\mathcal{M})$ constraints $s$ to satisfy both conditions, then the reachable states after any execution of the controller $u(k)$ will be constraint to satisfies $\phi_{u(k)}^{safe}$ and it will satisfy $\phi_{u(k+1)}^{initial} \wedge \phi_{u(k+1)}^{safe}$ before execute the next controller $u(k+1)$. Therefore, the hybrid program of resulting plan $s$ is $\phi_{u(1)}^{initial} \wedge \phi_{u(1)}^{safe} \rightarrow [\alpha_{u(1)};?(\phi_{u(1)}^{safe});?( \phi_{u(2)}^{initial} \wedge \phi_{u(2)}^{safe});\alpha_{u(2)};...;\alpha_{u(K)}] \phi_{u(K)}^{safe}$.
	The following theorem formally proves that the composition of those motion primitives will also guarantee the safety properties $\forall k \in \mathcal{N}: \phi_{u(k)}^{safe}$, after executing a plan.	
	\begin{thm}
		If a plan $s$ has size $K$, and satisfies $s \vDash \phi_{\mathcal{P}}(\mathcal{M})$ and  $\bigwedge_{\forall i \in \mathcal{N}_{\mathcal{U}}} \phi_{u_i}^{initial} \wedge \phi_{u_i}^{safe} \rightarrow [\alpha_{u_i}] \phi_{u_i}^{safe}$ is valid, then it will also satisfy all safety properties $s \vDash \bigwedge_{k \in \mathcal{N}} \phi_{u(k)}^{safe}$, where $s(k) = \langle q_{r}(k), \pi(k) \rangle$  at instant $k \in \mathcal{N}$ and $u(k) \in \pi(k)$.
	\end{thm}
	\begin{proof}
		The d$\mathcal{L}$ formula of $s$ that satisfies the specification $\phi_{\mathcal{P}}(\mathcal{M})$ is $\phi_{u(1)}^{initial} \wedge \phi_{u(1)}^{safe} \rightarrow [\alpha_{u(1)};?(\phi_{u(1)}^{safe});?( \phi_{u(2)}^{initial} \wedge \phi_{u(2)}^{safe});\alpha_{u(2)};...;\alpha_{u(K)}] \phi_{u(K)}^{safe}$. By applying the rules $[;]$, $[?]$ and $\rightarrow r$ \cite{platzer2010logical}, we find the equivalent formula $\bigwedge_{\forall k \in \mathcal{N}} \phi_{u(k)}^{initial} \wedge \phi_{u(k)}^{safe} \rightarrow [\alpha_{u(k)}] \phi_{u(k)}^{safe}$. Therefore, if we ensure that $\forall i \in \mathcal{N}_{\mathcal{U}}: \phi_{u_i}^{initial} \wedge \phi_{u_i}^{safe} \rightarrow [\alpha_{u_i}] \phi_{u_i}^{safe}$, the run $s$ will satisfy safety property of all motion primitives.
	\end{proof}
	
	Finally, the global layer encodes the CLTLB($\mathcal{D}$) task specifications $\phi_{\mathcal{G}}$ and $\phi_{\mathcal{P}}(\mathcal{M})$ into forms that are solvable by an SMT solver such as Z3\cite{de2008z3}, in the Constraint Generator shown in Fig. \ref{fig:framework}. Encoding the ITMP problem using CLTLB($\mathcal{D}$) language allows the description of a wide range of system properties that the satisfiability problem is decidable. If the global layer specification is satisfiable, a plan $s$ is then generated and extracted to the Motion Primitives Supervisor that enforces a sequential execution of this plan. In the following sections we will describe our framework in detail using the Clean Up example.

\section{Motion Primitives}\label{sec:prim}
In the Clean Up scenario, four motion controllers are needed so $\mathcal{U}=\{u_1,...,u_4\}$ where $u_1=$ go to, $u_2=$ push the door, $u_3=$ pick up and $u_4=$ leave. The following subsections will introduce how to design their motion primitives.

\subsection{Go To}

The first controller $u$ to be designed is the local navigation function which avoids obstacles that can be moving at a velocity up to $V$. Since the global layer does not take into account the  environment kinematics, the safety property must be verified at a local layer. We implement a Dynamic Window Approach \cite{fox1997dynamic} (DWA) algorithm based on the verification presented by Mitsch et. al. \cite{mitsch2013provably}.  

At every cycle time, based on the robot's sensor readings about its current position and surrounding obstacles, the DWA uses circular trajectories determined uniquely by the robot translational $v_r$ and rotational $\omega_r$ velocities. In summary, the algorithm is organized in two steps. (i) First it searches for a range of admissible $(v_r, \omega_r)$ pair that results in safe trajectories that the robot can realize in a short time frame, which is called dynamic window. (ii) Then, it chooses a $(v_r, \omega_r)$ pair in the dynamic window that maximizes the progress towards the goal.

The safety property that the DWA must satisfy is called Passive Safety Property $\phi_{ps}$\cite{macek2009towards}. This property means that the robot will never collide with the obstacle, or it will stop before collision. 
\begin{equation*}\label{eq:dwa_passive_safety_prop}
\phi_{ps} \equiv \Big(v_r = 0\Big) \vee \Big(\parallel p_r - p_o \parallel > \frac{v_r^2}{2b} + V\frac{v_r}{b} \Big)
\end{equation*}
where $b$ is the maximum deceleration, and $p_r$ and $p_o$ are the position of the robot and the obstacle, respectively.

Finally,  Mitsch et. al. \cite{mitsch2013provably} verified if this model ensures the Passive Safety Property $\phi_{ps}$ using KeYmaera. 

\begin{thm}
	If the DWA algorithm modeled with the hybrid program $dw_{ps}$ starts in a state that satisfies $\phi_{ps}$, it will always satisfies it.
	\begin{equation*}
	\phi_{ps} \rightarrow [dw_{ps}] \phi_{ps}
	\end{equation*}	
	where $dw_{ps}$ is presented in the Model 1 in \cite{mitsch2013provably}.	
\end{thm}
%

An example of such trajectories is sketched in Fig. \ref{fig:dwa_trajectory}. One robot passes in front of other robot executing the DWA algorithm. The DWA assigns circular trajectories to avoid the collision with the other robot, and the translational velocity is reduced based on the proximity to this robot. 

\begin{figure}[!t]
	\centering
	\includegraphics[width=1.5in]{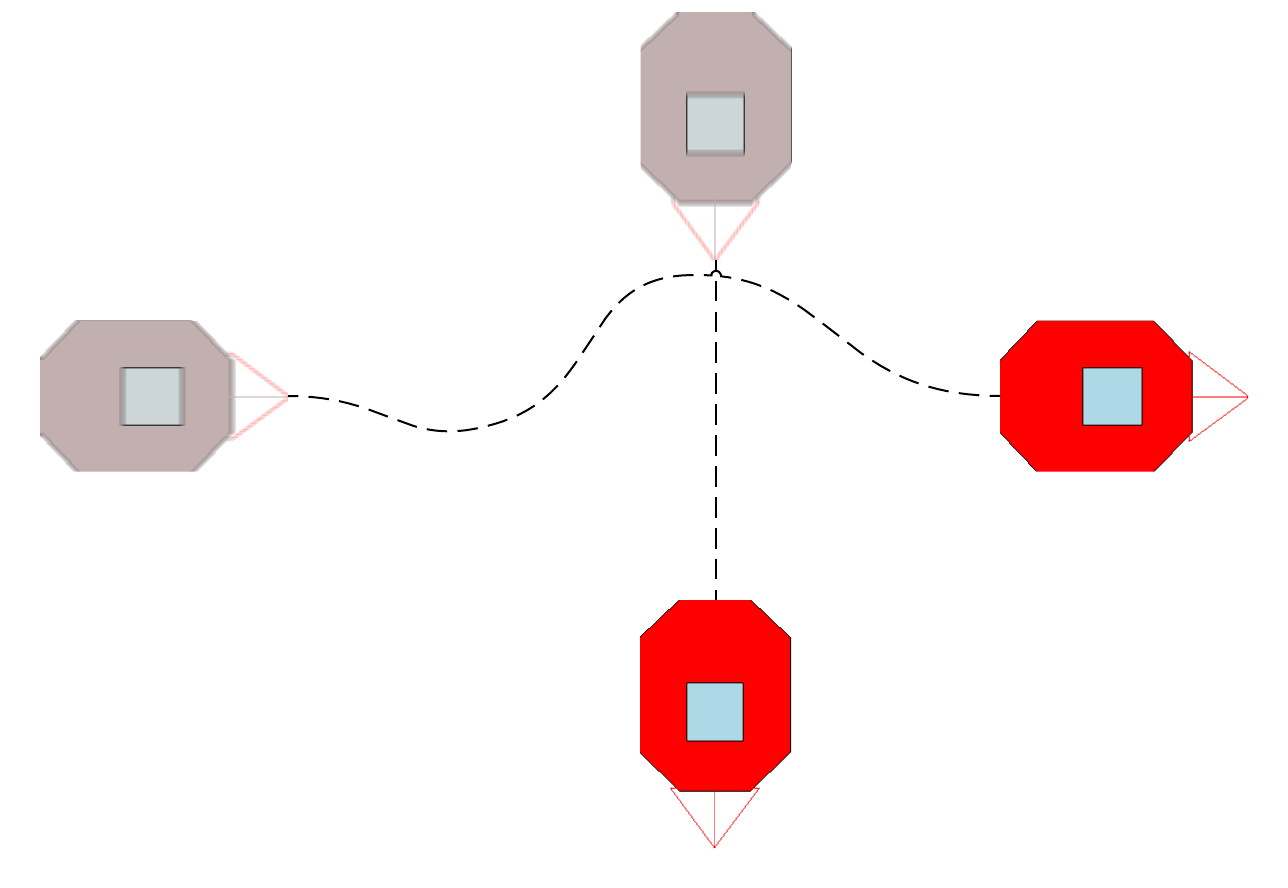}
	\caption{An example of safe circular trajectory for moving obstacles. The vehicles drive from the gray to the red in the dashed line, and the vehicle moving in straight line is the moving obstacle and the other is an executing DWA algorithm. }
	\label{fig:dwa_trajectory}
\end{figure}

Now we move from the local layer to the global layer to abstract this controller. Since the safety property $\phi_{ps}$ is an invariant property as well, it can be used as a constraint to ensure the safe motion primitive composition. Details like moving obstacles are omitted in the global layer, so we can assume that the minimum robot velocity is zero ($v_r > 0$), and the known obstacles are static ($V = 0$), thus:

\begin{corollary}\label{cor:corridor}
	If there is a trajectory between the current and next states ($q_r$ and $\bigcirc q_r$), in which the robot can fits into, then there is, at least, one possible DWA passively safe trajectory.
\end{corollary}
\begin{proof}
	If $v_r > 0$ and $V = 0$, thus, if $\parallel p_r - p_o \parallel > 0$, then there exist at least one trajectory that can be executed by the DWA algorithm that $\phi_{ps}$ holds true. 
\end{proof}

Hence, the Go To specification $\phi_{\Pi}$ in CLTLB($\mathcal{D}$) should guarantee that there exists a trajectory free of obstacles if the initial $q_r$ and goal $\bigcirc q_r$ states should be to the left, right, below or above of all obstacles (i.e. $r_{left,o}^{j} \equiv \Big( \max (\bigcirc q_r.x, q_r.x) \leq \min (o_j.x_i, o_j.x_f) - \frac{a.l}{2} \Big)$, $r_{right,o}^{j} \equiv \Big( \min (\bigcirc q_r.x, q_r.x) \geq \max (o_j.x_i, o_j.x_f) + \frac{a.l}{2} \Big)$, $r_{below,o}^{j} \equiv \Big( \max (\bigcirc q_r.y, q_r.y) \leq \min (o_j.y_i, o_j.y_f) - \frac{a.l}{2} \Big)$, $r_{above,o}^{j} \equiv \Big( \min (\bigcirc q_r.y, q_r.y) \geq \max (o_j.y_i, o_j.y_f) + \frac{a.l}{2} \Big)$),
\begin{align*}
\phi_{GoTo}^{\mathcal{O}} \equiv \forall j \in \mathcal{N}_{\mathcal{O}} : & \Box \Big[ (\pi = GoTo) \rightarrow  \\
& r_{left,o}^{j} \vee r_{right,o}^{j} \vee r_{bellow,o}^{j} \vee r_{above,o}^{j} \Big].
\end{align*}


And similarly we have $\phi_{GoTo}^{\mathcal{B}}$ to avoid colliding into objects that are not being carried (i.e. $\neg q_b^{j}.p$). Thus, the initial $q_r$ and goal $\bigcirc q_r$ states should be to the left, right, below or above of all objects (i.e. $r_{left,b}^{j} \equiv \Big( \max (\bigcirc q_r.x, q_r.x) \leq q_b^{j}.x - l^{j} \Big)$, $r_{right,b}^{j} \equiv \Big( \min (\bigcirc q_r.x, q_r.x) \geq q_b^{j}.x + l^{j} \Big)$, $r_{below,b}^{j} \equiv \Big( \max (\bigcirc q_r.y, q_r.y) \leq q_b^{j}.y - l^{j} \Big)$, $r_{above,b}^{j} \equiv \Big( \min (\bigcirc q_r.y, q_r.y) \geq q_b^{j}.y + l^{j} \Big)$, where $l^{j} = \frac{b_j.l+a.l}{2}$),
\begin{align*}
\phi_{GoTo}^{\mathcal{B}} & \equiv \forall j \in \mathcal{N}_{\mathcal{B}} : \Box \Big[ (\pi = GoTo) \wedge \neg q_b^{j}.p \rightarrow  \\
& r_{left,b}^{j} \vee r_{right,b}^{j} \vee r_{bellow,b}^{j} \vee r_{above,b}^{j} \Big].
\end{align*}


Additionally, the robot can only move inside the workspace (i.e. $rin_{x} \equiv (\mathcal{W}.x - \frac{\mathcal{W}.l}{2} + \frac{a.l}{2} \leq \bigcirc q_r.x \leq \mathcal{W}.x + \frac{\mathcal{W}.l}{2} - \frac{a.l}{2})$ and $rin_{y} \equiv (\mathcal{W}.y - \frac{\mathcal{W}.l}{2} + \frac{a.l}{2} \leq \bigcirc q_r.y \leq \mathcal{W}.y + \frac{\mathcal{W}.l}{2} - \frac{a.l}{2})$) and  won't change any object state (i.e. $p_{static}^{l} \equiv \bigcirc q_b^{l}.p = q_b^{l}.p$ when executing $Go To$, so, we have,
\begin{align*}
\phi_{GoTo} \equiv & \Box \Big[ \pi = GoTo \rightarrow \bigwedge_{l \in \mathcal{N}_{\mathcal{B}}} p_{static}^{l} \wedge rin_{x} \wedge rin_{y} \Big] \wedge \\ 
& \phi_{GoTo}^{\mathcal{O}} \wedge \phi_{GoTo}^{\mathcal{B}}
\end{align*}


\subsection{Push the Door}
Another reactive motion controller $u$ is to push the door, a straight movement in the direction of the door until it is pushed and the robot completely passes through it. The safety property $\phi_{Push}$ is that the robot must start at the initial position and go to final position (i.e. $push^j \equiv \Big(q_r = d_j.q_1 \wedge \bigcirc q_r.x = d_j.q_2.x \wedge \bigcirc q_r.y = d_j.q_2.y\Big) \vee \Big(q_r = d_j.q_2 \wedge \bigcirc q_r.x = d_j.q_1.x \wedge \bigcirc q_r.y = d_j.q_1.y\Big) $), so, we have,
\begin{align*}
\phi_{Push} \equiv \forall j \in  \mathcal{N}_{\mathcal{D}}: \Box \Big[ & \pi = Push_j \rightarrow \bigwedge_{l \in \mathcal{N}_{\mathcal{B}}} p_{static}^{l} \wedge push^j \Big]
\end{align*}

\subsection{Pick up and Leave}
Finally, the pick up and leave motion primitives describe the robot and objects dynamics. The safety property for those controllers does not depend on the robot or environment kinematics, so it does not require a verification in KeYmaera neither. So, we assume that the robot can pick up the object with the posing at $0\degree$. Hence, to pick an object up, the robot cannot be carrying any object (i.e. $\neg q_b^{l}.p$) and will carry the object $j$ (i.e. $\bigcirc p_{carry}^{j,l} \equiv (j=l \rightarrow \bigcirc q_b^{l}.p) \wedge (j \neq l \rightarrow \neg \bigcirc q_b^{l}.p)$). Also, the robot states will not change (i.e. $r_{static} \equiv (q_r.x = \bigcirc q_r.x) \wedge (q_r.y = \bigcirc q_r.y) \wedge (q_r.{\alpha} = \bigcirc q_r.{\alpha})$) and it will be posing in front of object (i.e. $r_{object}^{j} \equiv (q_r.\alpha = 0) \wedge (q_r.y = q_b^{j}.y) \wedge (q_r.x = q_b^{j}.x - l^j)$),
\begin{align*}
\phi_{PickUp} \equiv & \forall j \in \mathcal{N}_{\mathcal{B}}: \Box \Big[ \pi = PickUp_j \rightarrow \\
& \bigwedge_{\forall l \in \mathcal{N}_{\mathcal{B}}}\Big(\neg q_b^{l}.p \wedge \bigcirc p_{carry}^{j,l}\Big) \wedge r_{static} \wedge r_{object}^{j}\Big].
\end{align*}

Correspondingly, we leave the object at the same angle. Thus, to drop an object off, the robot should be carrying the object $j$ (i.e. $p_{carry}^{l} \equiv (j=l \rightarrow q_b^{l}.p) \wedge (j \neq l \rightarrow \neg q_b^{l}.p)$) and then drops it off (i.e. $\neg \bigcirc q_b^{l}.p$). Moreover, the robot will not change the initial and final states (i.e. $r_{static}^i$) and the object will be left next to it at $0^o$ (i.e. $b_{left}^{j} \equiv (q_r.\alpha = 0.0) \wedge (\bigcirc q_b^{j}.y = q_r.y) \wedge (\bigcirc q_b^{j}.x = q_r.x + l^j)$). However, we cannot leave the object over other objects. Therefore, the next object position should not have overlap with any other objects (i.e. $b_{left,b}^{j,l} \equiv \Big( \bigcirc q_b^{j}.y \leq \bigcirc q_b^{l}.y - l_b^{j,l} \Big)$, $b_{right,b}^{j,l} \equiv \Big( \bigcirc q_b^{j}.y \geq \bigcirc q_b^{l}.y + l_b^{j,l} \Big)$, $b_{below,b}^{j,l} \equiv \Big( \bigcirc q_b^{j}.x \leq \bigcirc q_b^{l}.x - l_b^{j,l} \Big)$, $b_{above,b}^{j,l} \equiv \Big( \bigcirc q_b^{j}.x \geq \bigcirc q_b^{l}.x + l_b^{j,l} \Big)$, where $l_b^{j,l} = \frac{b_j.l+b_l.l}{2}$),
\begin{align*}
\phi_{Leave}^{\mathcal{B}} \equiv & \forall j,l \in \mathcal{N}_{\mathcal{B}}, j \neq l : \Box \Big[ \pi = Leave_j \wedge \neg q_b^{l}.p \rightarrow \\ 
& b_{left,b}^{j,l} \vee b_{right,b}^{j,l} \vee b_{below,b}^{j,l} \vee b_{above,b}^{j,l} \Big].
\end{align*}

And neither over an obstacle, (i.e. $b_{left,o}^{j,l} \equiv \Big( \bigcirc q_b^{j}.y \leq \min (o.x_i, o.x_f) - \frac{b_j.l}{2} \Big)$, $b_{right,o}^{j,l} \equiv \Big( \bigcirc q_b^{j}.y \geq \max (o.x_i, o.x_f) + \frac{b_j.l}{2} \Big)$, $b_{below,o}^{j,l} \equiv \Big( \bigcirc q_b^{j}.x \leq \min (o.y_i, o.y_f) - \frac{b_j.l}{2} \Big)$, $b_{above,o}^{j,l} \equiv \Big( \bigcirc q_b^{j}.x \geq \max (o.y_i, o.y_f) + \frac{b_j.l}{2} \Big)$),
\begin{align*}
\phi_{Leave}^{\mathcal{O}} \equiv & \forall j \in \mathcal{B}, l \in \mathcal{O} : \Box \Big[ \pi = Leave_j \rightarrow \\
& b_{left,o}^{j,l} \vee b_{right,o}^{j,l} \vee b_{below,o}^{j,l} \vee b_{above,o}^{j,l} \Big].
\end{align*}

Hence,
\begin{align*}
\phi_{Leave} \equiv & \forall j \in \mathcal{N}_{\mathcal{B}}: \Box \Big[ \pi = Leave_j \rightarrow \\
& \bigwedge_{\forall l \in \mathcal{N}_{\mathcal{B}}} \Big(p_{carry}^{j,l} \wedge \neg \bigcirc q_b^{l}.p \Big) \wedge r_{static} \wedge b_{left}^{j}\Big] \wedge \\
& \phi_{Leave}^{\mathcal{B}} \wedge \phi_{Leave}^{\mathcal{O}}.
\end{align*}

Finally, if $\pi \neq Leave_j$, the object does not change position (i.e. $b_{static}^j \equiv (\bigcirc q_b^{j}.x = q_b^{j}.x) \wedge (\bigcirc q_b^{j}.y = q_b^{j}.y) \Big)$).
\begin{align*}
\phi_{carry} \equiv \forall j \in \mathcal{N}_{\mathcal{B}} : \Box \Big[ \Big( (\pi \neq Leave_j) \rightarrow b_{static}^j \Big].
\end{align*}


\section{Composition of motion primitives}\label{sec:comp}
The Motion Primitive Specifications $\phi_{\mathcal{P}}(\mathcal{M})$ are the conjunctions of the specifications from each single motion primitive. For the Clean Up scenario, this specification is,
\begin{align*}
\phi_{\mathcal{P}}^{Clean}(\mathcal{M}) \equiv & \phi_{GoTo} \wedge \phi_{Push} \wedge \phi_{PickUp} \wedge \phi_{Leave_b} \wedge \phi_{carry}
\end{align*} 
Now we can compose the motion primitives by encoding the Task Specification $\phi_{\mathcal{G}}$ and the Motion Primitive Specifications $\phi_{\mathcal{P}}(\mathcal{M})$ to Z3 SMT solver. If the specifications are satisfiable, the SMT solver will output a feasible plan $s$. 

We encode only CLTLB($\mathcal{D}$) formulas with no nested path quantifiers, but it is possible to encode nested formulas as well \cite{bersani2010bounded}. Since Z3 is a decision procedure for the combination of quantifier-free first-order logic with theories for linear arithmetic \cite{de2008z3}, we encode each state variable as an array whose size depends on the length of the trace $K$. For example, each robot state $q_r.x(k) \in Q_r$ will be encoded as an array such that $q_r.x[k]$. Each object state is a two dimensional array such that each element is $q_b[j].x[k]: j \in \mathcal{N}_{\mathcal{B}}$. Additionally, each motion primitives $\pi(k) \in Q_{\pi}$ will be an array such that each element is $\pi[k]$. Hence, the \textit{a.t.t.} operator $\bigcirc$ can be encoded by adding or subtracting the array index, for instance, $\bigcirc q_r.x \equiv q_r.x[k+1]$ at instant $k$. Therefore, a state formula $\psi$, defined as $\psi \equiv p \mid R(\varphi_1, \varphi_2,..., \varphi_n) \mid \neg \psi \mid \psi_1 \wedge \psi_2$, can be encoded to quantifier-free first-order logic formulas $\Psi[k]$, where $k \in \mathcal{N}_{\rho}$ is the instant that $\psi$ holds true. For instance, if $\psi \equiv q_b^0.p$, then $\Psi[2]$ holds true if $q_b^0.p$ holds true at instant $2$.  

Encoding temporal logic quantifiers to first order logic requires quantifiers $\forall$ and $\exists$ in relation to the time instants. The quantifier $\forall k \in \mathcal{N}_{\rho} : \Psi[k]$ can be implemented using for loop. The $\exists k \in \mathcal{N}_{\rho} : \Psi[k]$ can be encoded by using an auxiliary variable $j$ such as $\forall k \in \mathcal{N}_{\rho}: (k = j) \rightarrow \Psi[k] \wedge j \in \mathcal{N}_{\rho}$ and, then, also encoded using a for loop. Therefore, we can encode CLTLB($\mathcal{D}$) quantifiers to Z3, for example,

\begin{itemize}
	\item $\bigcirc^j \psi \Longleftrightarrow j \in \mathcal{N}_{\rho} \wedge \Psi[j]$
	\item $\psi_1 \mathbf{U} \psi_2 \Longleftrightarrow \begin{cases}
	\Big(\bigwedge_{k \in \mathcal{N}_{\rho}} \Big[(k < j \rightarrow \Psi_1[k]) \wedge \\
	(k = j \rightarrow \Psi_2[k])\Big] \wedge j \in \mathcal{N}_{\rho}
	\end{cases}$
	\item $\Box \psi \Longleftrightarrow \bigwedge_{k \in \mathcal{N}_{\rho}} \Psi[k]$
	\item $\diamondsuit \psi \Longleftrightarrow \bigwedge_{k \in \mathcal{N}_{\rho}} \Big[k = j \rightarrow \Psi[k]\Big] \wedge j \in \mathcal{N}_{\rho}$
	\item $Last [\psi] \Longleftrightarrow \Psi[K]$
\end{itemize}

Finally, let $\varphi_1$ and $\varphi_2$ be \textit{a.t.t.}'s, the functions $\max(\varphi_1,\varphi_2)$ and $\min(\varphi_1,\varphi_2)$ are encoded with SMT function $ite$, i.e. $\max(\varphi_1,\varphi_2) \equiv ite(\varphi_1 > \varphi_2, \varphi_1,\varphi_2)$ and $\min(x,y) \equiv ite(\varphi_1 < 
\varphi_2,\varphi_1,\varphi_2)$. Now, we can define a task specification in CLTLB($\mathcal{D}$) and find an integrated task and motion plan $s$ for the scenario in the Example \ref{ex:example01} as shown below.

\begin{example}\label{ex:example02}
	A task can be any temporal logic describing how the robot or the objects should move in the environment. For example, it could be to bring the objects in the workspace of Fig. \ref{fig:example01} to the temporary rectangular area $Q_{temp}^j \equiv (-1500 \leq q_b^j.x \leq -500) \wedge (-2500 \leq q_b^j.y \leq -1000) \wedge \neg q_b^j.p$ (which can be represented by the coordinates of its upper-left and lower-right vertices) and later leave them in the state $Q_b^{goal} \equiv q_b^1 = \langle 1900, 1000, false \rangle \wedge q_b^2 = \langle 2000, 1000, false \rangle$. The specification of this task is,
	\begin{align*}
	\phi_{\mathcal{G}}^{Clean} \equiv \bigwedge_{\forall j \in N_{\mathcal{B}}} \diamondsuit \Big[  Q_{temp}^j \Big] \wedge Last \Big[ Q_b^{goal} \Big]
	\end{align*}
	
	If we set the trace length $K = 24$ and encode the formula $\phi_{\mathcal{G}}^{Clean} \wedge \phi_{\mathcal{P}}^{Clean}(\mathcal{M})$ to a SMT solver, we can find the following satisfiable plan,
{\small		\begin{align*}
		s = \{&\langle \Pi_1, (-2000, -500, 0) \rangle, \langle \Pi_2, (-1000, -500, 0) \rangle,\\
		& \langle \Pi_1, (1650, -1000, 0) \rangle, \langle \Pi_3, (1650, -1000, 0) \rangle, \\ 
		& \langle \Pi_1, (-998, -1251, 0) \rangle, \langle \Pi_4, (-998, -1251, 0) \rangle, \\ 
		& \langle \Pi_1, (-999, -999, 0) \rangle, \langle \Pi_1, (1750, -1000, 0) \rangle, \\ 
		& \langle \Pi_3, (1750, -1000, 0) \rangle, \langle \Pi_1, (-751, -1001, 0) \rangle, \\ 
		& \langle \Pi_4, (-751, -1001, 0) \rangle, \langle \Pi_3, (-751, -1001, 0) \rangle, \\ 
		& \langle \Pi_1, (-1000, -500, 90) \rangle, \langle \Pi_2, (-1000, 500, 90) \rangle, \\ 
		& \langle \Pi_1, (1750, 1000, 0) \rangle, \langle \Pi_4, (1750, 1000, 0) \rangle, \\ 
		& \langle \Pi_1, (-1000, 500, 270) \rangle, \langle \Pi_2, (-1000, 500, 270) \rangle, \\ 
		& \langle \Pi_1, (-998, -1251, 0) \rangle, \langle \Pi_3, (-998, -1251, 0) \rangle, \\ 
		& \langle \Pi_1, (-1000, -500, 90) \rangle, \langle \Pi_2, (-1000, 500, 90) \rangle, \\ 
		& \langle \Pi_1, (1650, 1000, 0) \rangle, \langle \Pi_4, (1650, 1000, 0) \rangle	\}
		\end{align*}
		where $\Pi_1=$ go to, $\Pi_2=$ push the door, $\Pi_3=$ pick up and $\Pi_4=$ leave motion primitives.}
%
%
%
	
\end{example}

	Note that this plan is safe to moving obstacles as well because the motion primitive Go To can handle it in the local layer. For example, when executing $\langle \Pi_1, (-2000, -500, 0) \rangle$, if a human appears moving inside the robot straight trajectory to position $(-2000, -500)$, it will reduce the velocity properly and try another circular trajectory that does not leads to a collision, shown in Fig. \ref{fig:dwa_trajectory}. If the robot cannot avoid the collision, it is formally proven that it will be stopped before it. Therefore, if the moving obstacle can and chooses to avoid the collision too, the robot will be always safe. It allows the robot always to find a new plan $s$ in a receding horizon strategy. Hence, if the moving obstacles change the environment in a way that the plan is not feasible, we can always update the scene description $\mathcal{M}$ to search for a new satisfiable plan at current state. 

\section{Simulation/Experimental Results}\label{sec:sim}

The purpose of the experiments in this section is to determine which parameters can affect the computation time of the framework. The benchmarks shown here are relevant to designing an optimization algorithm that finds the best path using CoSMoP motion planning. All the experiments were executed on Linux with an Intel i7 processor and 8GB memory. 

The environment is one floor of a building with square layout. It has multiple rooms with push-pull doors that permit the robot to move between rooms, shown in the Fig. \ref{fig:rooms}. The robot starts in the room marked with S, and it needs to reach the room with G. All doors permit the robot to move in the direction of the goal. To increase the complexity of the environment, we can increase its size and the number of rooms. It is executed 35 times for each scene description. The time average and standard variation are then calculated.

\begin{figure}[!t]
	\centering
	\includegraphics[width=1.5in]{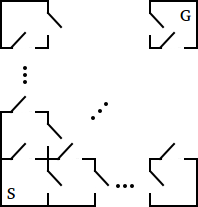}
	\caption{Environment framework used in the execution time experiments.}
	\label{fig:rooms}
\end{figure}

To increase the environment size does not seem to affect the execution time, as shown in table \ref{tab:increasingsz}. However, it expressively increases when we raise the number of rooms, see table \ref{tab:increasingroom}. It suggests that the environment complexity to solve the trajectory is only affected by the number of obstacles and doors in the workspace. Moreover, we can assert that the precision on the robot state integer variables won't change the execution time, for example, if we change the position precision from $mm$ to $\mu m$. The size of the trace $K$ also increases the execution time as it increases, as shown in the table \ref{tab:increasingk}. Although, the number of rooms (i.e. obstacles and doors) can affect the size of satisfiable trace $K$, the table \ref{tab:increasingcomplex} shows that CoSMoP finds a reachable trajectory in a reasonable time for a significantly complex Scene Description. 

\begin{table}
	\caption{Benchmarks increasing the environment size.}
	\label{tab:increasingsz}
	\begin{tabular}{|c|c|c|c|}
		\hline Environment (m) & $\#$ Rooms & $K$ & Time (avg $\pm$ std) \\ \hline
		$4 \times 4$ & 9 & 14 & $30.2ms \pm 0.0067$ \\ 
		$8 \times 8$ & 9 & 14 & $30.9ms \pm 0.0057$ \\ 
		$16 \times 16$ & 9 & 14 & $30.2ms \pm 0.0088$ \\ 
		$32 \times 32$ & 9 & 14 & $30.6ms \pm 0.019$ \\ 
		$64 \times 64$ & 9 & 14 & $30.7ms \pm 0.0052$ \\ 
		$128 \times 128$ & 9 & 14 & $30.1ms \pm 0.056$ \\ 
		$256 \times 256$ & 9 & 14 & $30.0ms \pm 0.0047$ \\ \hline
	\end{tabular} 
\end{table}

\begin{table}
	\caption{Benchmarks increasing the number of rooms.}
	\label{tab:increasingroom}
	\begin{tabular}{|c|c|c|c|}
		\hline Environment (m) & $\#$ Rooms & $K$ & Time (avg $\pm$ std) \\ \hline
		$32 \times 32$ & 9 & 50 & $121.5ms \pm 0.020$ \\ 
		$32 \times 32$ & 25 & 50 & $434.3ms \pm 0.13$ \\ 
		$32 \times 32$ & 49 & 50 & $1007.0ms \pm 0.94$ \\ 
		$32 \times 32$ & 81 & 50 & $3994.1ms \pm 12$ \\ \hline
	\end{tabular} 
\end{table}

\begin{table}
	\caption{Benchmarks increasing the size of the trace $K$.}
	\label{tab:increasingk}
	\begin{tabular}{|c|c|c|c|}
		\hline Environment (m) & $\#$ Rooms & $K$ & Time (avg $\pm$ std) \\ \hline
		$32 \times 32$ & 25 & 26 & $181.6ms \pm 3.0$ \\ 
		$32 \times 32$ & 25 & 32 & $284.9ms \pm 0.031$ \\ 
		$32 \times 32$ & 25 & 38 & $290.4ms \pm 0.027$ \\ 
		$32 \times 32$ & 25 & 44 & $344.4ms \pm 0.034$ \\ 
		$32 \times 32$ & 25 & 50 & $369.1ms \pm 0.032$ \\ \hline
	\end{tabular} 
\end{table}

\begin{table}
	\caption{Benchmarks increasing the problem complexity.}
	\label{tab:increasingcomplex}
	\begin{tabular}{|c|c|c|c|}
		\hline Environment (m) & $\#$ Rooms & $K$ & Time (avg $\pm$ std) \\ \hline
		$32 \times 32$ & 9 & 14 & $34.7ms \pm 0.0051$ \\ 
		$32 \times 32$ & 16 & 20 & $69.3ms \pm 0.026$ \\ 
		$32 \times 32$ & 25 & 26 & $298.2ms \pm 0.28$ \\ 
		$32 \times 32$ & 36 & 32 & $461.6ms \pm 0.59$ \\ 
		$32 \times 32$ & 49 & 38 & $856.2ms \pm 0.63$ \\ 
		$32 \times 32$ & 64 & 44 & $1126.9ms \pm 0.26$ \\ 
		$32 \times 32$ & 81 & 50 & $3909.3ms \pm 1.7$ \\ \hline
	\end{tabular} 
\end{table}

\section{Conclusion}\label{sec:conclusion}

We proposed the CoSMoP, an ITMP approach using formal bottom-up design for mobile robot planning. It synthesizes a sequential execution of motion primitives that ensures the task specification and safety properties even under a dynamic environment with moving obstacles. The main advantage of our approach is that we can handle the moving obstacles dynamics and some uncertainties about the environment at local motion primitive, which increase its robustness. Additionally, we evaluated CoSMoP in a motivating example showing that the ITMP can synthesize a correct plan, we also studied how different parameters affect the execution time. Future works includes implementation on real robot and extensions to multiple robots.  

%
%


\small
\bibliographystyle{plain}

\bibliography{cosmop}

\end{document}